%% file: main.tex
\documentclass[runningheads]{llncs}
\usepackage[T1]{fontenc}

\usepackage{times}
\usepackage{latexsym}
\usepackage{amsmath}
\usepackage{amssymb}
\usepackage{xcolor}
\usepackage{colortbl}

\usepackage{lineno}
\usepackage{hyperref} 

\usepackage{algorithm} 
\usepackage[algo2e]{algorithm2e} 
\usepackage{pifont,dsfont}
\usepackage{makecell}
\usepackage{amsfonts}
\usepackage{wrapfig,lipsum,booktabs}
\usepackage[export]{adjustbox}
\usepackage{xcolor,colortbl}

\usepackage{graphicx}
\usepackage{subfigure}
\usepackage{paralist}
\usepackage{refcount}
\usepackage{amsfonts}
\usepackage{caption}
\usepackage{makecell} 

\usepackage{algorithm}
\usepackage{algpseudocode}

\newcommand{\myModel}{\textsc{HyQNet}}

\title{
Neural-Symbolic Logic Query Answering in Non-Euclidean Space
}

\titlerunning{Neural-Symbolic Logic Query Answering in Non-Euclidean Space}

\author{
Lihui Liu\inst{1}
}

\authorrunning{Lihui Liu.}

\institute{
Wayne State University, Detroit, Michigan, USA \\
\email{\{hw6926\}@wayne.edu} \\
}

\begin{document}

\maketitle

\begin{abstract}
\input{000abs.tex}

\vspace{-0.7\baselineskip}
\end{abstract}


\section{Introduction}
\input{001intro.tex}

\section{Problem Definition}
\input{003problem}


\section{Methodology}\label{single}
\input{004method}

\section{Experiment}

\input{005experiment}

\section{Related Work}\label{preliminary}
\input{002related_work}

\section{Conclusion}\label{preliminary}

\input{008conclusion}

\bibliographystyle{splncs04}
\bibliography{008reference,liu}


\end{document}

%% file: 000abs.tex
Answering complex first-order logic (FOL) queries on knowledge graphs is essential for reasoning. Symbolic methods offer interpretability but struggle with incomplete graphs, while neural approaches generalize better but lack transparency. Neural-symbolic models aim to integrate both strengths but often fail to capture the hierarchical structure of logical queries, limiting their effectiveness.
We propose \myModel, a neural-symbolic model for logic query reasoning that fully leverages hyperbolic space. \myModel\ decomposes FOL queries into relation projections and logical operations over fuzzy sets, enhancing interpretability. To address missing links, it employs a hyperbolic GNN-based approach for knowledge graph completion in hyperbolic space, effectively embedding the recursive query tree while preserving structural dependencies.
By utilizing hyperbolic representations, \myModel\ captures the hierarchical nature of logical projection reasoning more effectively than Euclidean-based approaches. Experiments on three benchmark datasets demonstrate that \myModel\ achieves strong performance, highlighting the advantages of reasoning in hyperbolic space.

%% file: 001intro.tex
A knowledge graph organizes entities and their relationships, making it a valuable tool for applications like question answering ~\cite{gpt2}, recommendation systems~\cite{guo2020surveyknowledgegraphbasedrecommender}, and computer vision\cite{he2015deepresiduallearningimage}. Reasoning over knowledge graphs helps infer new knowledge and answer queries based on existing data, with answering complex First-Order Logic (FOL) queries being a key challenge.  
FOL queries involve logical operations such as existential quantifiers (\( \exists \)), conjunction (\( \wedge \)), disjunction (\( \vee \)), and negation (\( \neg \)). For example, the question ``Which universities do Turing Award winners in deep learning work at?" can be formulated as an FOL query, as shown in Figure.~\ref{fig:example}.  

Most existing methods rely on neural networks to model logical operations and learn query embeddings to find answers. Approaches like \cite{GQE,query2box,betaE,fuzzyQE} represent queries using boxes, beta distributions, or points, but their reasoning process is often opaque since neural networks function as black boxes. On the other hand, symbolic methods provide clear interpretability by explicitly deriving answers from stored facts or using subgraph matching. However, they struggle with incomplete knowledge graphs, which limits their applicability in real-world scenarios.

Recent research has explored combining neural reasoning with symbolic techniques to achieve both strong generalization and interpretability. For example, GNN-QE \cite{gnnqe} employs a graph neural network (GNN) for projection operations and integrates symbolic graph matching to enhance robustness and interpretability. However, GNN-QE encodes local subtrees in Euclidean space, which is suboptimal for capturing hierarchical relationships. Hyperbolic embeddings provide a more natural representation for such structures.  

In this paper, we propose a hyperbolic GNN for answering complex First-Order Logic (FOL) queries on knowledge graphs. Unlike Euclidean-based approaches, our model leverages hyperbolic embeddings with learned curvature at each layer, enabling it to more effectively capture hierarchical query structures. Building upon prior work \cite{gnnqe}, we decompose FOL queries into expressions over fuzzy sets, where relation projections are modeled using the hyperbolic GNN with learned curvature. This design improves reasoning accuracy.
By operating in hyperbolic space, our model naturally preserves the tree-like structure of logical projections, facilitating more effective representation learning. Additionally, the learned curvature allows the GNN to adapt to the geometry of the knowledge graph, improving entity projections while maintaining interpretability through fuzzy set-based logic operations.
We evaluate our method on three standard datasets for FOL queries. Our experiments demonstrate that our approach consistently outperforms existing baselines across various query types, achieving state-of-the-art performance on all datasets.

In summary, we make the following contributions:  
\begin{itemize}  
    \item We propose a novel method that leverages hyperbolic embeddings with learned curvature to better capture hierarchical structures in logical projections. By learning curvature at each layer, our model gains increased flexibility, leading to improved reasoning performance. 
    \item We demonstrate that our method achieves state-of-the-art performance on three standard FOL query datasets.  
\end{itemize}

%% file: 003problem.tex
In this section, we introduce the background knowledge of FOL queries on knowledge graphs and fuzzy sets.

\subsection{First-Order Logic Queries on Knowledge Graphs}  
Given a set of entities $V$ and a set of relations $R$, a knowledge graph $G = (V, E, R)$ is a collection of triplets $E = \{(h_i, r_i, t_i)\} \subseteq V \times R \times V$, where each triplet is a fact from head entity $h_i$ to tail entity $t_i$ with the relation type $r_i$.  

A FOL query on a knowledge graph is a formula composed of constants (denoted with English terms), variables (denoted with $a$, $b$, $c$), relation symbols (denoted with $R(a, b)$), and logic symbols ($\exists$, $\land$, $\lor$, $\lnot$). In the context of knowledge graphs, each constant or variable is an entity in $V$. A variable is \textit{bounded} if it is quantified in the expression, and \textit{free} otherwise. Each relation symbol $R(a, b)$ is a binary function that indicates whether there is a relation $R$ between a pair of constants or variables. For logic symbols, we consider queries that contain conjunction ($\land$), disjunction ($\lor$), negation ($\lnot$), and existential quantification ($\exists$)\footnote{For simplicity, we focus on this subset of FOL operations in this work.}.  

Figure~\ref{fig:example} illustrates the FOL query for the natural language question “where did Canadian citizens with Turing Award graduate?”. Given a FOL query, the goal is to find answers to the free variables, such that the formula is true.  

\begin{figure}[hbt!]
	\centering
	\includegraphics[width=0.7\textwidth]{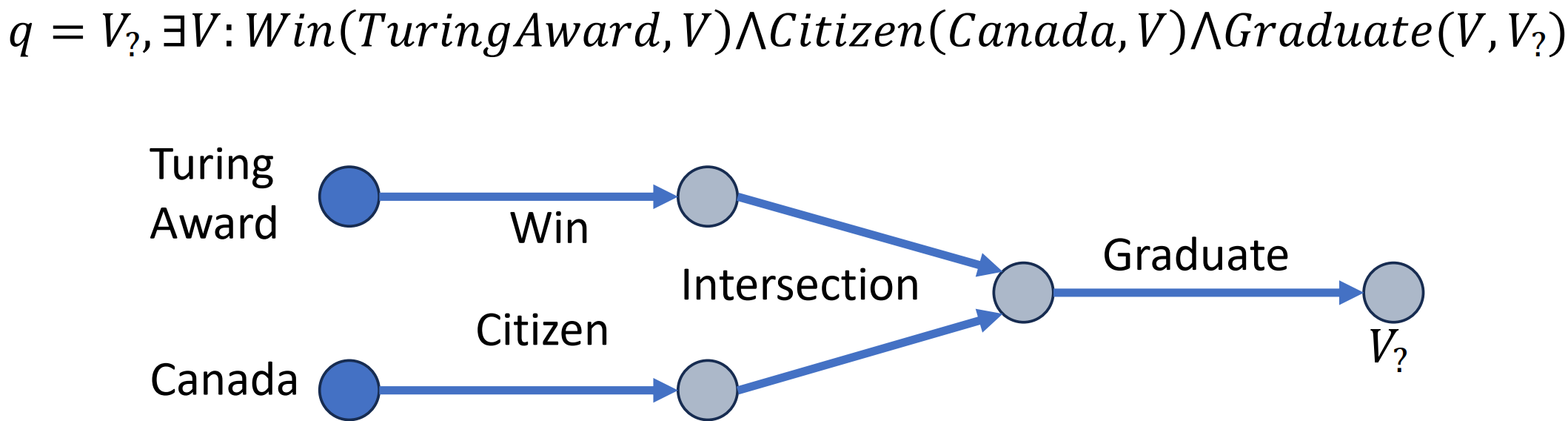}
	\caption{An example of logical query.}
    \label{fig:example}
\end{figure}

\subsection{Fuzzy Logic Operations}  
Fuzzy sets \cite{fuzzyset} are a continuous relaxation of sets whose elements have degrees of membership. A fuzzy set $A = (U, x)$ contains a universal set $U$ and a membership function $x : U \to [0, 1]$. For each $u \in U$, the value of $x(u)$ defines the degree of membership (i.e., probability) for $u$ in $A$.  
Similar to Boolean logic, fuzzy logic defines three logic operations: \textit{AND}, \textit{OR}, and \textit{NOT}, over the real-valued degree of membership. There are several alternative definitions for these operations, such as product fuzzy logic, Gödel fuzzy logic, and Łukasiewicz fuzzy logic.  

In this paper, fuzzy sets are used to represent the assignments of variables in FOL queries, where the universe $U$ is always the set of entities $V$ in the knowledge graph. Since the universe is a finite set, we represent the membership function $x$ as a vector $x$. We use $x_u$ to denote the degree of membership for element $u$. For simplicity, we abbreviate a fuzzy set $A = (U, x)$ as $x$ throughout the paper.

\subsection{Hyperbolic Space and Its Embedding}  

Hyperbolic space is a geometric setting where the parallel postulate of Euclidean geometry does not hold. Instead of parallel lines remaining equidistant, they diverge exponentially. This space, often denoted as $H^n$, exhibits a constant negative curvature, leading to unique geometric properties such as rapid expansion of volume with distance and altered notions of angles and distances. These properties make hyperbolic space particularly useful for representing hierarchical structures and relational data, where entities grow exponentially as one moves outward.  

\textbf{Mathematical Representation via the Poincaré Model.}  
One of the most practical ways to work with hyperbolic space is through the Poincaré ball model, which maps it onto the interior of a unit sphere in $\mathbb{R}^n$. Given a point $x$ within this unit ball, distances are not measured using standard Euclidean norms but rather through a modified metric that accounts for hyperbolic curvature. Specifically, the hyperbolic distance $d(x, y)$ between two points $x, y$ is computed as:  
\[
d(x, y) = \text{arccosh} \left(1 + 2 \frac{\|x - y\|^2}{(1 - \|x\|^2)(1 - \|y\|^2)} \right).
\]  
This metric ensures that distances grow exponentially as points move outward, mirroring the behavior of many complex networks and hierarchical systems. 


\textbf{Tangent Space and Exponential Map.}  
Although hyperbolic space is fundamentally different from Euclidean space, computations can be simplified by leveraging the tangent space at a given point. The tangent space at a point $x \in D^n$ is a local Euclidean approximation of hyperbolic space, allowing for operations such as optimization and vector transformations. The exponential map $\exp_x: T_xH^n \to H^n$ projects vectors from the tangent space back onto the manifold, while its inverse, the logarithmic map $\log_x: H^n \to T_xH^n$, maps points from the manifold to the tangent space:  
\[
\exp_x(v) = x + \tanh(\|v\|) \frac{v}{\|v\|}
\]
\[
\log_x(y) = \frac{\text{arctanh}(\|y - x\|)}{\|y - x\|} (y - x).
\]  


%% file: 004method.tex
\begin{figure*}[ht]
    \centering
    \includegraphics[width=0.99\textwidth]{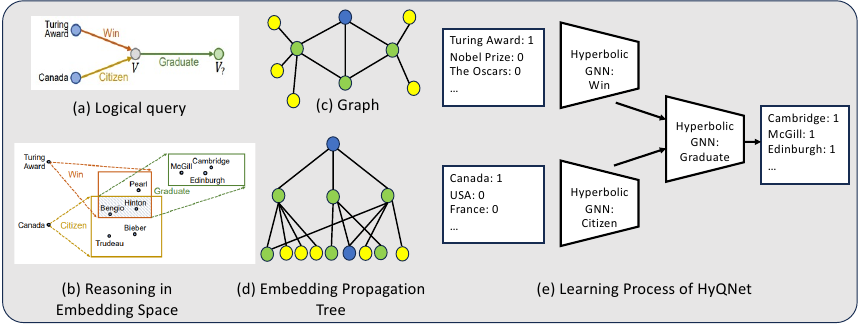}
    \caption{Overview of \myModel.}
    \label{fig:framework}
\end{figure*}

Here we present our model \myModel. The high-level idea of \myModel\ is to first decompose a FOL query into an expression of four basic operations (relation projection, conjunction, disjunction, and negation) over fuzzy sets, 
then parameterize the relation projection with a hyperbolic GNN with learned curvature adapted from KG completion, 
and instantiate the logic operations with product fuzzy logic operations. 

Given a FOL query, the first step is to transform it into a structured computation process, where the query is broken down into a sequence of fundamental operations. We represent the query as a directed computation graph, starting from anchor entities and progressively applying transformations until the final answer set is obtained.  
The decomposition follows a structured process. First, relation projections identify the entities that satisfy specific relational constraints. Given an initial fuzzy set of head entities, the projection operation maps them to potential tail entities through the corresponding relation. 
Once relation projections establish the core structure, logical operations refine the intermediate results. Conjunction enforces conditions that must be met simultaneously, effectively taking the intersection of fuzzy sets. Disjunction broadens the set of possible answers by considering alternative paths, while negation eliminates entities that contradict the query constraints.  
By iteratively applying these operations, we propagate information through the computation graph, ultimately producing a fuzzy set representing the final answer. This structured execution ensures interpretability at each step while maintaining the flexibility needed for reasoning in knowledge graphs.  
We introduce the details of each logical operation below. 

\textbf{4.2. Hyperbolic Neural Relation Projection}  
To solve complex queries on incomplete knowledge graphs, we train a neural model to perform relation projection, defined as \( y = P_q(x) \), where \( x \) is the input fuzzy set, \( q \) is the given relation, and \( y \) is the output fuzzy set.  
The goal of relation projection is to generate the potential answer set \( y \), consisting of entities that could be connected to \( x \) via relation \( q \). This task aligns with the knowledge graph completion problem, where missing links must be inferred.  
Specifically, the neural relation projection model aims to predict the fuzzy set of tail entities \( y \) given the fuzzy set of head entities \( x \) and the relation \( q \), even in the presence of incomplete knowledge.

Many methods have been proposed to model projection operations, including~\cite{GQE,query2box,betaE}. Recently, \cite{gnnqe} introduced a GNN-based framework for knowledge graph completion.  
When applying GNN models, it has been shown that their underlying mechanism is equivalent to learning node embeddings based on a recursive tree structure. An example is illustrated in Figure~\ref{fig:framework} (d), where the goal is to predict a suitable answer for the query \((\text{Blue node}, r, ?)\), where \( r \) represents an arbitrary relation. However, the corresponding recursive tree structure grows exponentially with the number of propagation steps.  
Although naively applying GNNs as in~\cite{gnnqe} is feasible, it is suboptimal because standard GNNs struggle to effectively capture the exponential growth of node neighborhoods. In contrast, hyperbolic embeddings are better suited for encoding hierarchical structures, as they naturally capture the exponential growth of node neighborhoods.  
To address this, we propose leveraging hyperbolic embeddings with learned curvature for relation projection. Inspired by this idea, we develop a scalable hyperbolic GNN framework to enhance the effectiveness of relation projection in knowledge graph completion.

\noindent\textbf{Hyperbolic Relation Projection.} Our goal is to design a hyperbolic GNN model that predicts a fuzzy set of tail entities given a fuzzy set of head entities and a relation. 
Given a head entity $u$ and a projection relation $q$, we use the following iteration to compute a representation $h_v$ for each entity $v \in V$ with respect to the source entity $u$ according to the recursive tree structure.
Given the source entity set $x_v$, we initalize its embedding by $h^{(0)}_v \leftarrow x_v q$,
where the $x_v$ is the probability of entity $v$ in $x$. 
For each layer in the tree, the learned process can be denoted as 
\begin{align*}
h^{(t)}_v &\leftarrow \text{AGGREGATE}( \{ \text{MESSAGE}(h^{(t-1)}_z (z, r, v)) \mid (z, r, v) \in E(v) \} )
\end{align*}

When utilizing hyperbolic embeddings, we iteratively update node embeddings in the knowledge graph. The update process is given by:  
\[
\begin{aligned}
h^{(t+1)}_v &= \sigma \left( \exp_c^0 \left( \tilde{A} \, \mathbf{W}_t \log_c^0 (h^{(t)}_z) \right) \right), 
&\quad (z, r, v) \in E(v).
\end{aligned}
\]
Here, the exponential map \(\exp_c^x\) maps a tangent vector \( v \in T_x D^n_c \) to the hyperbolic manifold:  
\[
\exp_c^x(v) = x \oplus_c \tanh\left(\sqrt{c} \frac{\|v\|}{\sqrt{c}}\right) \frac{v}{\|v\|}.
\]
Conversely, the logarithmic map \(\log_c^x\) transforms a point on the hyperbolic manifold back to the tangent space:  
\[
\log_c^x(y) = \frac{2 \sqrt{c} \, \text{arctanh}(\sqrt{c} \| -x \oplus_c y \|)}{\| -x \oplus_c y \|} (-x \oplus_c y).
\]
In practice, \(\exp_c^0\) and \(\log_c^0\) are used to efficiently convert between Euclidean vectors and their hyperbolic representations in the Poincaré ball model.

To apply the hyperbolic GNN framework for relation projection, we propagate the representations for $T$ layers. Then, we take the representations in the last layer and pass them into a multi-layer perceptron (MLP) $f$ followed by a sigmoid function $\sigma$ to predict the fuzzy set of tail entities: $P_q(x) = \sigma\left(f(h^{(T)}\right))$.

\subsection{Fuzzy Logic Operations}
In knowledge graph reasoning, operations like \( C(x, y) \), \( D(x, y) \), and \( N(x) \) play an essential role in combining the results of multiple relation projections, thus forming the input fuzzy set for the next projection. These operations must ideally satisfy foundational logical principles, such as commutativity, associativity, and non-contradiction. While many prior approaches \cite{GQE,query2box,betaE} have introduced geometric operations to model these logic operations within the embedding space, such neural operators are not always reliable in adhering to logical laws. As a result, chaining these operators can introduce errors.

Building on the work of \cite{gnnqe}, we employ product fuzzy logic operations to model conjunction, disjunction, and negation. Specifically, given two fuzzy sets \( x, y \in [0, 1]^V \), we define the operations as:

\[
    C(x, y) = x \odot y \tag{Conjunction}
\]
\[
    D(x, y) = x + y - x \odot y \tag{Disjunction}
\]
\[
    N(x) = 1 - x \tag{Negation}
\]

In these equations, \( \odot \) denotes element-wise multiplication, and \( 1 \) is a vector of ones (representing the universe). 

\subsection{Model Training}

In terms of model training, we follow the approach used in previous studies \cite{query2box,betaE}, aiming to minimize the binary cross-entropy loss function, defined as:

\begin{align*}
L = - \frac{1}{|A_Q|} \sum_{a \in A_Q} \log p(a \mid Q) 
- \frac{1}{|V \setminus A_Q|} \sum_{a' \in V \setminus A_Q} \log \big( 1 - p(a' \mid Q) \big)
\end{align*}

Here, \( A_Q \) refers to the set of correct answers for query \( Q \), while \( V \setminus A_Q \) represents the set of entities in the knowledge graph that are not part of \( A_Q \). The terms \( p(a \mid Q) \) and \( p(a' \mid Q) \) indicate the predicted probabilities for \( a \) and \( a' \) being correct answers, respectively.

\begin{theorem}
The proposed \myModel\ consistently outperforms—or performs on par with—existing GNN-based logical query reasoning models that rely on Euclidean embeddings.
\end{theorem}

\begin{proof}
Euclidean embeddings correspond to a special case of \myModel\ where the learned curvature approaches zero. As \myModel\ generalizes Euclidean embeddings by allowing variable curvature, it is capable of capturing a broader range of geometric structures. Consequently, \myModel\ consistently matches or surpasses the performance of models using Euclidean embeddings.
\end{proof}

%% file: 005experiment.tex
\begin{table*}[htbp]
\centering
\scriptsize
\setlength{\tabcolsep}{4pt}
\caption{Test MRR results (\%) on answering FOL queries. }
\begin{tabular}{|l|cccccccccccccccc|}
\hline
Model & ${avg}_p$ & ${avg}_n$ & 1p & 2p & 3p & 2i & 3i & pi & ip & 2u & up & 2in & 3in & inp & pin & pni \\ \hline
\multicolumn{17}{|c|}{FB15k}  \\ \hline
GQE & 28.0 & - & 54.6 & 15.3 & 10.8 & 39.7 & 51.4 & 27.6 & 19.1 & 22.1 & 11.6 & - & - & - & - & - \\ 
Q2B & 38.0 & - & 68.0 & 21.0 & 14.2 & 55.1 & 66.5 & 39.4 & 26.1 & 35.1 & 16.7 & - & - & - & - & -\\ 
BetaE & 41.6 & 11.8 & 65.1 & 25.7 & 24.7 & 55.8 & 66.5 & 43.9 & 28.1 & 40.1 & 25.2 & 14.3 & 14.7 & 11.5 & 6.5 & 12.4 \\ 
CQD-CO & 46.9 & - & 89.2 & 25.3 & 13.4 & 74.4 & 78.3 & 44.1 & 33.2 & 41.8 & 21.9 & - & - & - & - & - \\ 
CQD-Beam & 58.2 & - & \textbf{89.2} & 54.3 & 28.6 & 74.4 & 78.3 & 58.2 & 67.7 & 42.4 & 30.9 & - & - & - & - & - \\ 
GNN-QE & 73.7 & 38.3 & 87.7 & 68.8 & 58.7 & 79.7 & 83.5 & 68.9 & \textbf{68.1} & 74.4 & 60.2 & \textbf{44.5} & 41.7 & 41.7 & 29.4 & 34.3 \\  \hline
\myModel\ &\textbf{74.2} & \textbf{38.6} & 87.8 & \textbf{69.2} & \textbf{59.4} & \textbf{81.0} & \textbf{84.8} & \textbf{71.1} & 65.8 & \textbf{74.9} & \textbf{60.9} & 44.3 & \textbf{44.4} & \textbf{42.1} & \textbf{29.8} & \textbf{34.4} \\ \hline
\multicolumn{17}{|c|}{FB15k-237}  \\ \hline
GQE & 16.3 & - & 35.0 & 7.2 & 5.3 & 23.3 & 34.6 & 16.5 & 10.7 & 8.2 & 5.7 & - & - & - & - & - \\  
Q2B & 20.1 & - & 40.6 & 9.4 & 6.8 & 29.5 & 42.3 & 21.2 & 12.6 & 11.3 & 7.6 & - & - & - & - & - \\  
BetaE & 20.9 & 5.5 & 39.0 & 10.9 & 10.0 & 28.8 & 42.5 & 22.4 & 12.6 & 12.4 & 9.7 & 3.5 & 3.4 & 5.1 & 7.9 & 7.4 \\  
CQD-CO & 21.8 & - & 46.7 & 9.5 & 6.3 & 31.2 & 40.6 & 23.6 & 16.0 & 14.5 & 8.2 & - & - & - & - & - \\  
CQD-Beam & 22.3 & - & \textbf{46.7} & 11.6 & 8.0 & 31.2 & 40.6 & 21.2 & {18.7} & 14.6 & 8.4 & - & - & - & - & - \\  
FuzzQE & 24.0 & 7.8 & 42.8 & \textbf{12.9} & 10.3 & 33.3 & 46.9 & 26.9 & 17.8 & \textbf{14.6} & 10.3 & \textbf{8.5} & 11.6 & 7.8 & 5.2 & 5.8 \\  
GNN-QE &  26.1 & 8.4 & 40.1 & 12.1 & 10.2 & 35.5 & 51.6 & \textbf{29.2} & 16.9 & 13.4 & 11.2 & 7.7 & 14.4 & 8.3 & 5.7 & 5.9 \\  \hline
\myModel\ & \textbf{26.5} & \textbf{8.9} & 40.5 & 11.9 & \textbf{10.3} & \textbf{36.3} & \textbf{52.1} & \underline{28.8} & \textbf{17.9} & 14.3 & \textbf{11.2} & 7.6 & \textbf{15.8} & \textbf{8.9} & \textbf{6.3} & \textbf{5.9}  \\  \hline 
\multicolumn{17}{|c|}{NELL995}  \\ \hline
GQE & 18.6 & - & 32.8 & 11.9 & 9.6 & 27.5 & 35.2 & 18.4 & 14.4 & 8.5 & 8.8 & - & - & - & - & - \\ 
Q2B & 22.9 & - & 42.2 & 14.0 & 11.2 & 33.3 & 44.5 & 22.4 & 16.8 & 11.3 & 10.3 & - & - & - & - & - \\ 
BetaE & 24.6 & 5.9 & 53.0 & 13.0 & 11.4 & 37.6 & 47.5 & 24.1 & 14.3 & 12.2 & 8.5 & 5.1 & 7.8 & 10.0 & 3.1 & 3.5 \\ 
CQD-CO & 28.8 & - & 60.4 & 17.8 & 12.7 & 39.3 & 46.6 & \textbf{30.1} & 22.0 & 17.3 & \textbf{13.2} & - & - & - & - & - \\ 
CQD-Beam & 28.6 & - & \textbf{60.4} & \textbf{20.6} & 11.6 & 39.3 & 46.6 & 25.4 & \textbf{23.9} & \textbf{17.5} & 12.2 & - & - & - & - & - \\ 
FuzzQE & 27.0 & 7.8 & 47.4 & 17.2 & \textbf{14.6} & 39.5 & 49.2 & 26.2 & 20.6 & 15.3 & 12.6 & 7.8 & 9.8 & 11.1 & 4.9 & 5.5 \\ 
GNN-QE & 28.7 & 8.6 & 51.3 & 17.0 & 12.9 & 39.1 & 49.3 & 28.1 & 17.4 & 14.2 & 9.0 & \textbf{8.8} & 12.7 & 10.9  & 5.0 & 5.5 \\ \hline
\myModel\ - & \textbf{28.9} & \textbf{8.9} & 50.3 & \underline{18.0} & \underline{14.1} & \textbf{39.7} & \textbf{49.9} & 27.4 & 17.5 & 14.5 & 10.8 & \underline{8.5} & \textbf{13.5} & \textbf{11.4} & \textbf{5.8} & \textbf{5.5} \\ \hline
\end{tabular}
\label{tab:FOL_results_mrr}
\end{table*}

\begin{table*}[htbp]
\centering
\scriptsize
\setlength{\tabcolsep}{4pt}
\caption{Test H@1 results (\%) on answering FOL queries.}
\begin{tabular}{|l|c|c|c|c|c|c|c|c|c|c|c|c|c|c|c|c|}
\hline
Model & ${avg}_p$ & ${avg}_n$ & 1p & 2p & 3p & 2i & 3i & pi & ip & 2u & up & 2in & 3in & inp & pin & pni \\ \hline
\multicolumn{17}{|c|}{FB15k}  \\ \hline
GQE & 16.6 & - & 34.2 & 8.3 & 5.0 & 23.8 & 34.9 & 15.5 & 11.2 & 11.5 & 5.6 & - & - & - & - & - \\ 
Q2B & 26.8 & - & 52.0 & 12.7 & 7.8 & 40.5 & 53.4 & 26.7 & 16.7 & 22.0 & 9.4 & - & - & - & - & - \\ 
BetaE & 31.3 & 5.2 & 52.0 & 17.0 & 16.9 & 43.5 & 55.3 & 32.3 & 19.3 & 28.1 & 16.9 & 6.4 & 6.7 & 5.5 & 2.0 & 5.3 \\ 
CQD-CO & 39.7 & - & 85.8 & 17.8 & 9.0 & 67.6 & 71.7 & 34.5 & 24.5 & 30.9 & 15.5 & - & - & - & - & - \\ 
CQD-Beam & 51.9 & - & \textbf{85.8} & 48.6 & 22.5 & 67.6 & 71.7 & 51.7 & 62.3 & 31.7 & 25.0 & - & - & - & - & - \\ 
GNN-QE  & 68.6 & 21.7 & 85.4 & 63.5 & 52.5 & 74.8 & 79.9 & 63.2 & \textbf{62.5} & 67.1 & 53.0 & 32.3 & 30.9 & 32.7 & 17.8 & 21.8 \\ \hline
\myModel\ & \textbf{69.5} & \textbf{21.9} & 85.4 & \textbf{64.6} & \textbf{53.8} & \textbf{76.8} & \textbf{80.4} & \textbf{63.2} & 62.0 & \textbf{70.5} & \textbf{54.3} & \textbf{32.4} & \textbf{30.9} & \textbf{32.7} & \textbf{18.2} & \textbf{22.2} \\ \hline
\multicolumn{17}{|c|}{FB15k-237}  \\ \hline
GQE & 8.8 & - & 22.4 & 2.8 & 2.1 & 11.7 & 20.9 & 8.4 & 5.7 & 3.3 & 2.1 & - & - & - & - & - \\ 
Q2B & 12.3 & - & 28.3 & 4.1 & 3.0 & 17.5 & 29.5 & 12.3 & 7.1 & 5.2 & 3.3 & - & - & - & - & - \\ 
BetaE & 13.4 & 2.8 & 28.9 & 5.5 & 4.9 & 18.3 & 31.7 & 14.0 & 6.7 & 6.3 & 4.6 & 1.5 & 7.7 & 3.0 & 0.9 & 0.9 \\ 
CQD-CO & 14.7 & - & 36.6 & 4.7 & 3.0 & 20.7 & 29.6 & 15.5 & 9.9 & 8.6 & 4.0 & - & - & - & - & - \\ 
CQD-Beam & 15.1 & - & \textbf{36.6} & 6.3 & 4.3 & 20.7 & 29.6 & 13.5 & 12.1 & \textbf{8.7} & 4.3 & - & - & - & - & - \\ 
GNN-QE  & 18.4 & 3.2 & 29.9 & 6.4 & 5.5 & 24.6 & 41.5 & \textbf{20.7} & 11.2 & 7.8 & 6.2 & 2.7 & 6.1 & 3.6 & 1.9 & 1.9 \\ \hline
\myModel\ & \textbf{18.9} & \textbf{3.8} & 30.3 & \textbf{6.7} & \textbf{5.9} & \textbf{25.6} & \textbf{42.2} & 20.5 & \textbf{12.2} & 8.0 & \textbf{6.4} & \textbf{3.1} & \textbf{7.7} & \textbf{4.1} & \textbf{2.5} & \textbf{2.1} \\ \hline
\multicolumn{17}{|c|}{NELL995}  \\ \hline
GQE & 9.9 & - & 15.4 & 6.7 & 5.0 & 14.3 & 20.4 & 10.6 & 9.0 & 2.9 & 5.0 & - & - & - & - & - \\ 
Q2B & 14.1 & - & 23.8 & 8.7 & 6.9 & 20.3 & 31.5 & 14.3 & 10.7 & 5.0 & 6.0 & - & - & - & - & - \\ 
BetaE & 17.8 & 2.1 & 43.5 & 8.1 & 7.0 & 27.2 & 36.5 & 17.4 & 9.3 & 6.9 & 4.7 & 1.6 & 2.2 & 4.8 & 0.7 & 1.2 \\ 
CQD-CO & 21.3 & - & 51.2 & 11.8 & \textbf{9.0} & 28.4 & 36.3 & \textbf{22.4} & 15.5 & 9.9 & \textbf{7.6} & - & - & - & - & - \\ 
CQD-Beam & 21.0 & - & \textbf{51.2} & \textbf{14.3} & 6.3 & 28.4 & 36.3 & 18.1 & \textbf{17.4} & \textbf{10.2} & 7.2 & - & - & - & - & - \\ 
GNN-QE & 21.5 & 3.3 & 41.0 & 12.4 & 9.3 & \textbf{29.2} & \textbf{40.0} & 20.2 & 11.7 & 8.2 & 6.4 & 2.5 & 5.3 & 5.5 & 1.5 & 1.7 \\ \hline
\myModel\ & \textbf{21.5} & \textbf{3.7} & 42.3 & 12.0 & 8.4 & \underline{28.6} & 39.4 & \underline{21.1} & 12.1 & 7.7 & 5.4 & \textbf{3.1} & \underline{4.6} & \textbf{8.1} & \underline{0.9} & \textbf{1.9} \\ \hline

\end{tabular}
\label{tab:FOL_results_hits}
\end{table*}

We evaluate \myModel\ on three widely used knowledge graph datasets: FB15k, FB15k-237, and NELL995 \cite{gnnqe}. To ensure consistency with prior studies, we adopt the benchmark FOL queries provided by BetaE \cite{betaE}, which include 9 EPFO query structures and 5 additional queries incorporating negation. Our model is trained on 10 query types (1p, 2p, 3p, 2i, 3i, 2in, 3in, inp, pni, pin), following the setup used in previous research \cite{query2box,betaE,gnnqe}. For evaluation, we test \myModel\ on both the 10 training query types and 4 unseen query types (ip, pi, 2u, up) to assess its generalization ability. 

\subsubsection{Evaluation Protocol}
We follow the evaluation protocol of \cite{query2box}, where the answers to each query are divided into two categories: easy answers and hard answers. For test (or validation) queries, easy answers are entities that can be directly reached in the validation (or train) graph using symbolic relation traversal. In contrast, hard answers require reasoning over predicted links, meaning the model must infer them rather than retrieve them directly. To evaluate performance, we rank each hard answer against all non-answer entities and use mean reciprocal rank (MRR) and HITS@K (H@K) as evaluation metrics. Specifically, we use MRR and HITS@1, as done in GNN-QE.

\subsubsection{Baselines}
We compare \myModel\ against both embedding methods and neural-symbolic methods. The embedding methods include GQE \cite{GQE}, Q2B \cite{query2box}, BetaE \cite{betaE} and FuzzQE \cite{fuzzyQE}. The neural-symbolic methods include CQD-CO, CQD-Beam \cite{CQD} and GNN-QE ~\cite{gnnqe}.

\subsection{Complex Query Answering}

\begin{table*}[htbp]
\centering
\scriptsize
\setlength{\tabcolsep}{3pt}
\caption{Spearman’s rank correlation between the model prediction and the number of ground truth answers on FB15k. avg is the average correlation on all 12 query types in the table.}
\begin{tabular}{|l|c|c|c|c|c|c|c|c|c|c|c|c|c|c|c|c|c|c|c|}
\hline
Model & avg & 1p & 2p & 3p & 2i & 3i & pi & ip & 2in & 3in & inp & pin & pni \\ \hline
Q2B & - & 0.301 & 0.219 & 0.262 & 0.331 & 0.270 & 0.297 & 0.139 & - & - & - & - & -  \\ \hline
BetaE & 0.494 & 0.373 & 0.478 & 0.472 & 0.572 & 0.397 & 0.519 & 0.421 & 0.622 & 0.548 & 0.459 & 0.465 & 0.608 \\ \hline
GNN-QE & 0.952 & 0.971 & 0.967 & 0.926 & 0.987 & 0.936 & 0.937 & 0.923 & 0.992 & 0.985 & 0.880 & 0.940 & 0.991 \\ \hline
\myModel\ & \textbf{0.958} & \textbf{0.971} & \textbf{0.970} & \textbf{0.941} & \textbf{0.988} & \textbf{0.938} & \textbf{0.939} & \textbf{0.934} & 0.990 & 0.984 & \textbf{0.908} & \textbf{0.948} & \textbf{0.991} \\ \hline

\end{tabular}
\label{tab:cardinality}
\end{table*}

Table ~\ref{tab:FOL_results_mrr} shows the MRR results of different models for answering FOL queries. GQE, Q2B, CQD-CO, and CQD-Beam do not support queries with negation, so the corresponding entries are empty. 
We observe that \myModel\ achieves the best result most of the time for both EPFO queries and queries with negation on all 3 datasets. Notably, \myModel\ achives the best average performance and better overall performance most of the time compared with baselines, which shows the effectiveness of the proposed hyperbolic embedding model. Hyperbolic embedding can better capture the recursive learning tree structure, Thus, the learned embedding is more expressive compared with the euclidean embedding.  
Additionally, \myModel\ consistently achieves the highest average MRR, particularly excelling in complex queries such as intersection (2i, 3i) and negation-based queries (2in, 3in, pin, pni). This suggests that these models effectively capture relational structures and logical patterns. {CQD-Beam} also performs well, especially on FB15k, benefiting from beam search strategies that enhance multi-hop reasoning. However, these gains are less pronounced in FB15k-237 and NELL995, likely due to the increased sparsity and relational diversity in these datasets. Overall, the results indicate that neural-symbolic models (e.g., GNN-QE, \myModel) and query embedding-based approaches (e.g., CQD variants) are more effective at generalizing across different logical structures compared to earlier embedding-based methods like GQE and Q2B.

Table ~\ref{tab:FOL_results_hits} shows the Hits@1 (H@1) results of different methods on FOL queries, demonstrating a consistent pattern of model performance across various datasets (FB15k, FB15k-237, and NELL-995). {\myModel} consistently outperforms other models, achieving the highest average Hits@1 scores across all datasets.
{CQD-Beam} also performs well, particularly on the FB15k dataset, benefiting from beam search strategies that enhance multi-hop reasoning. In contrast, models such as {GQE} and {Q2B} exhibit relatively lower performance, especially on more complex queries and larger datasets like FB15k-237 and NELL995. 
Compared to GNN-QE, \myModel\ shows superior performance. Overall, the results suggest that the hyperbolic embedding model is better at capturing the recursive tree structure in embeddings.

\subsection{Answer Set Cardinality Prediction}
Similar to GNN-QE, \myModel\ can predict the cardinality of the answer set (i.e., the number of answers) without explicit supervision. The cardinality of a fuzzy set is computed as the sum of entity probabilities exceeding a predefined threshold. We use 0.5 as the threshold, as it naturally aligns with our binary classification loss.
Previous studies \cite{query2box,betaE} have observed that the uncertainty in Q2B and BetaE is positively correlated with the number of answers. Following their approach, we report Spearman’s rank correlation between our model’s predictions and the ground truth. As shown in Table ~\ref{tab:cardinality}, \myModel\ significantly outperforms existing methods most of the time.



%% file: 002related_work.tex
\textbf{Knowledge Graph Reasoning.}  
Knowledge graph reasoning has been studied for a long time
~\cite{liu2019g,liu2021neural,liu2021kompare,liu2022joint,liu2022comparative,liu2022knowledge,liu2023knowledge,liu2016brps,liu2024logic,liu2024can,liu2024new,liu2024conversational,liu2025transnet,liu2025neural,liu2025few,liu2025monte,liu2025hyperkgr,liu2025mixrag,liu2024knowledge,liu2025unifying,liu2026neural,liuneural,liu2026accurate,liu2026accurate2,liu2026ambiguous,liu2026ambiguous2}
. Embedding-based approaches \cite{bordes2013translating,rotate,complex} map entities and relations into low-dimensional vectors to capture graph structure. Reinforcement learning methods \cite{deeppath,ZHANG2022102933,LinRX2018_MultiHopKG} use agents to explore paths and predict missing links, while rule-based techniques \cite{ho2018rule,yang2017differentiable,rnnlogic} generate logical rules for link prediction.

\textbf{Graph Neural Networks for Relational Learning.}  
Graph neural networks (GNNs) \cite{grail,gnnqe} have gained popularity in learning entity representations for knowledge graph completion. In contrast, our approach adapts hyperbolic GNNs for relation projection and extends the task to complex logical query answering, which introduces more challenges than traditional graph completion.

\textbf{Hyperbolic Geometry in Knowledge Graphs.}  
Hyperbolic space is particularly effective for modeling hierarchical structures due to its exponential growth properties. Early works such as Poincaré embeddings \cite{hyper1} demonstrated the benefits of hyperbolic geometry, which were later expanded in Hyperbolic Graph Neural Networks (HGNNs) \cite{hypergnn} for better relational reasoning. Models like MuRP \cite{hyperMuRp} and AttH \cite{hyperAttH} have shown improved performance in hierarchical link prediction tasks, making hyperbolic embeddings valuable for relational reasoning in large-scale knowledge graphs. Our work builds on these models to optimize hyperbolic representations for logical reasoning and query answering.

\textbf{Complex Logical Query}  
\textbf{Handling Complex Logical Queries.}  
Traditional knowledge graph completion techniques focus on predicting missing links, but more complex tasks involve answering logical queries that combine operations like conjunction, disjunction, and negation. Methods such as GQE \cite{GQE} introduced compositional training to extend embedding-based approaches for path queries. GQE applies a geometric intersection operator for conjunctive queries ($\land$), while Query2Box \cite{query2box} and BetaE \cite{betaE} generalize this approach to handle more complex logical operators like $\lor$ and $\lnot$. FuzzQE \cite{fuzzyQE} incorporates fuzzy logic to align embeddings with classical logic, providing a more nuanced representation. However, these approaches often prioritize computational efficiency, using nearest neighbor search or dot product for query decoding, which sacrifices interpretability and makes intermediate reasoning opaque.

\textbf{Combining Symbolic and Neural Approaches.}  
To address the limitations of pure embedding methods, some approaches combine neural networks with symbolic reasoning. CQD \cite{CQD} extends pretrained embeddings to answer complex queries, with CQD-CO employing continuous optimization and CQD-Beam utilizing beam search. While our method shares similarities with CQD-Beam in combining symbolic reasoning with graph completion, it stands apart by enabling direct training on complex queries. Unlike CQD-Beam, GNN-QE bypasses the need for exhaustive search or pretrained embeddings, offering a more efficient and interpretable solution for answering complex logical queries.

%% file: 008conclusion.tex
We proposed a hyperbolic graph neural network (GNN) for answering complex First-Order Logic (FOL) queries on knowledge graphs. By leveraging hyperbolic embeddings with learned curvature, our model effectively captures hierarchical query structures, improving reasoning accuracy and interpretability.  
Experiments on three benchmark datasets show that our method outperforms existing baselines, demonstrating the benefits of hyperbolic space for logical reasoning. Future work includes extending our approach to more expressive logical operations and exploring adaptive curvature learning for greater flexibility.